\def\year{2020}\relax
\documentclass[letterpaper]{article} 
\usepackage{aaai20}  
\usepackage{times}  
\usepackage{helvet} 
\usepackage{courier}  
\usepackage[hyphens]{url}  
\usepackage{graphicx} 
\usepackage{graphicx}  
\usepackage{tikz}
\usepackage{tabstackengine}
\usepackage{amssymb,amsfonts}
\usepackage[all,arc]{xy}
\usepackage{enumerate}
\usepackage{mathrsfs}
\usepackage{amsthm}
\usepackage{stackengine}
\usepackage{subcaption}
\usepackage[ruled,vlined]{algorithm2e}
\usepackage{todonotes}

\newcommand{\R}{\mathbb{R}}

\newcommand{\citet}[1]{\citeauthor{#1} \shortcite{#1}} 
\newcommand{\citep}{\cite} 

\newtheorem{thm}{Theorem}

\newtheorem{prop}[thm]{Proposition}

\newtheorem{Def}[thm]{Definition}

\usepackage{mathtools}


\makeatletter
\renewcommand*\env@matrix[1][\arraystretch]{%
  \edef\arraystretch{#1}%
  \hskip -\arraycolsep
  \let\@ifnextchar\new@ifnextchar
  \array{*\c@MaxMatrixCols c}}
\makeatother

\usepackage{tikz}
\usetikzlibrary{decorations.pathreplacing,calc}

\usepackage{algpseudocode}
\algnewcommand{\algorithmicgoto}{\textbf{go to}}%
\algnewcommand{\Goto}[1]{\algorithmicgoto~\ref{#1}}%

\title{Deep Message Passing on Sets}
\author{Yifeng Shi , Junier Oliva , Marc Niethammer\Large \\ 
Department of Computer Science, UNC-Chapel Hill, USA\\
\{yifengs, joliva, mn\}@cs.unc.edu
}
\begin{document}

\maketitle

\begin{abstract}
Modern methods for learning over graph input data have shown the fruitfulness of accounting for relationships among elements in a collection. However, most methods that learn over set input data use only rudimentary approaches to exploit intra-collection relationships. In this work we introduce \textit{Deep Message Passing on Sets} (DMPS), a novel method that incorporates relational learning for sets. DMPS not only connects learning on graphs with learning on sets via \textit{deep kernel learning}, but it also bridges message passing on sets and traditional diffusion dynamics commonly used in denoising models. Based on these connections, we develop two new blocks for relational learning on sets: the \textit{set-denoising block} and the \textit{set-residual block}. The former is motivated by the connection between message passing on general graphs and diffusion-based denoising models, whereas the latter is inspired by the well-known residual network. In addition to demonstrating the interpretability of our model by learning the true underlying relational structure experimentally, we also show the effectiveness of our approach on both synthetic and real-world datasets by achieving results that are competitive with or outperform the state-of-the-art.
\end{abstract}

\section{Introduction}
\label{sec:introduction}

Significant effort in machine learning has been devoted to methods that operate over fixed-length, finite vectors. These methods ultimately perform some variant of a classic functional estimation task where one maps one fixed input vector $x \in \R^d$ to another fixed output vector $y \in \R^p$ via an estimated function $\hat{f}: \R^d \mapsto \R^p$. Notwithstanding the impressive progress of these approaches, the world we live in is filled with data that does not come neatly pre-packaged into fixed finite vectors. Instead, often data is observed and reasoned over in collections such as sets. For instance, when performing object detection on point clouds, one assigns a label (the object type) based on the underlying shape that is inferred collectively using all observed points (as opposed to labelling any one individual 3d point). In these, and many other tasks, one seeks to assign an output response to an \emph{entire} set of elements in which the elements can themselves be related to each other in a complex manner. Based on this need for analysis approaches that can operate on sets relationally, we develop a machine learning (ML) estimation technique that incorporates relational learning for set-structured data. 

Set-structured data in its general form presents fundamental challenges to many existing machine learning methods. First, an appropriate method should be invariant to the order of the elements in the input set; i.e., different permutations should not influence the final response since the underlying instance is an \emph{unordered} set. Second, a method should allow input sets of variable cardinalities; i.e., we should be able to associate a single label with a variable number of set elements. Both of these challenges render traditional approaches based on ordered inputs of fixed dimension (e.g., vectors or images of given sizes) like standard multilayer perceptrons (MLP) inapplicable. 
Some rectifications have been proposed to cope with these challenges without directly addressing them. For example, one can train a recurrent neural network (RNN) to adapt to variable-sized input sets. 
However, there is no guarantee that a network can learn to be permutation invariant~\citep{vinyals}, especially in the context of input sets with large cardinality and small sample sizes. Hence, in this work we instead consider deep learning (DL) architectures that are specifically constructed to handle set-structured data.


\begin{figure*}[!t]
\centering
\begin{subfigure}{.473\textwidth}
  \centering
  \includegraphics[width=1\linewidth]{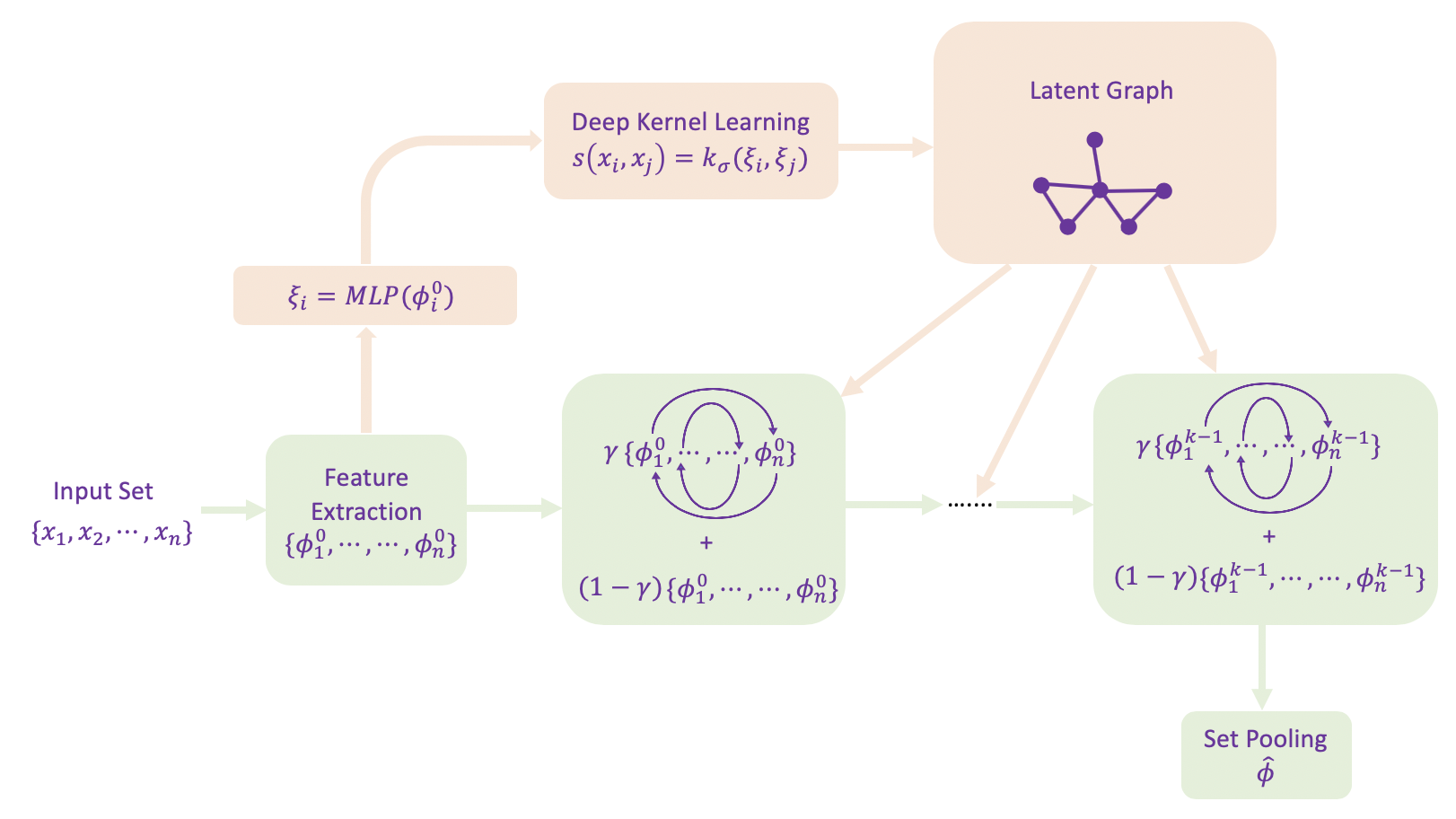}
  \caption{DMPS with the Set-denoising Block}
  \label{fig:dmps_overview}
\end{subfigure}%
\begin{subfigure}{.265\textwidth}
  \centering
  \includegraphics[width=1\linewidth]{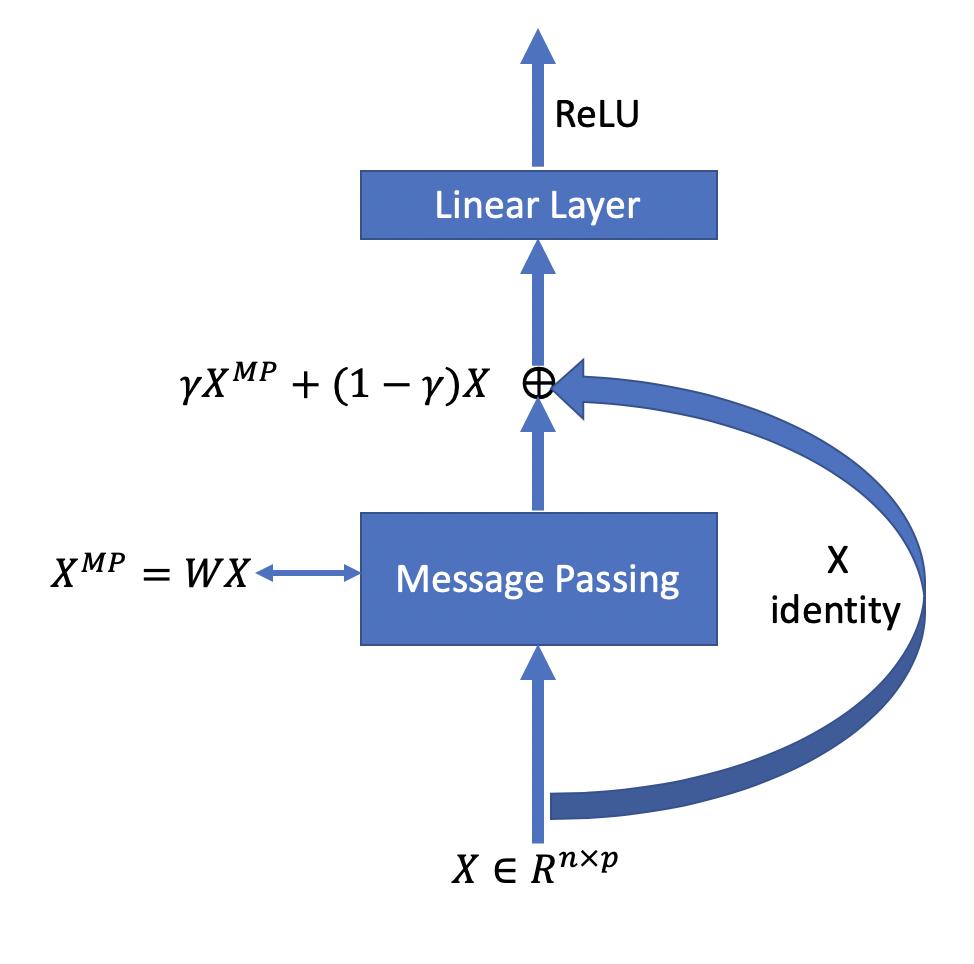}
  \caption{Set-denoising Block}
  \label{fig:denoising}
\end{subfigure}%
\begin{subfigure}{.287\textwidth}
  \centering
  \includegraphics[width=1\linewidth]{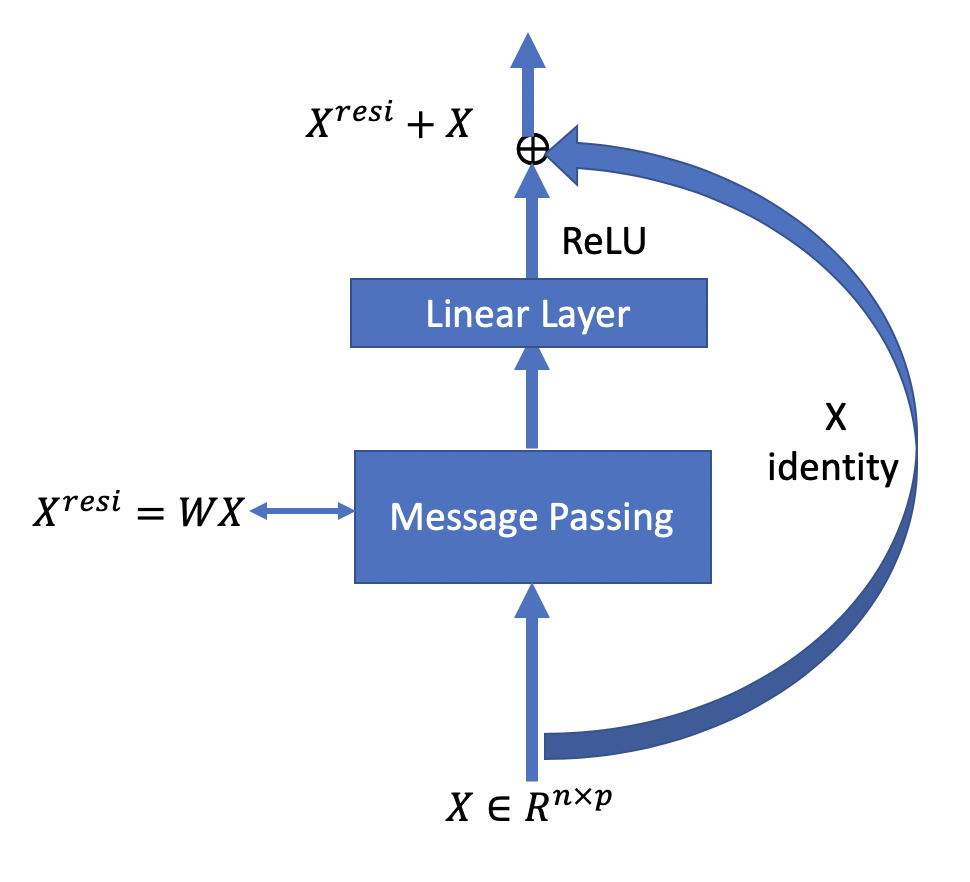}
  \caption{Set-residual Block}
  \label{fig:residual}
\end{subfigure}
\caption{An overview of DMPS coupled with the set-denoising block as an illustrative example, and magnified views of the set-denoising block and the set-residual block. We briefly explain each figure here. (a): Given a set of objects, $\{x_1, x_2, \cdots, x_n\}$, we first extract feature representations of the set elements, $\{\phi^0_1, \phi^0_2, \cdots, \phi^0_n \}$, through, for example, a convolution neural network if the set elements are images. We estimate the underlying latent graph of the set using the extracted features via deep kernel learning \cite{wilson}. We then apply a predefined number ($k$, in this case) of set-denoising blocks, which can be replaced by the message passing step or the set-residual block, to the extracted features in order to produce meaningful final feature representations that encode complex interactions among set elements. Lastly, we use a set pooling operation to generate a set-level feature representation for downstream tasks; (b): We first process the feature matrix, $X$, with a message passing step to produce $X^{MP}$, and then compute a weighted sum between $X$ and $X^{MP}$ with a learnable diffusion coefficient $\gamma$. Lastly we process the added matrix through a linear layer followed by a non-linear operator; (c): We still process the feature matrix, $X$, with a message passing step to produce $X^{resi}$. Instead of adding $X^{resi}$ immediately back to $X$, we process $X^{resi}$ with a linear layer followed by a non-linear operator first, and then add the resulting feature matrix to $X$.}
\label{fig:overall}
\end{figure*}

Architectures that handle set-structured data exist. These techniques use global-pooling operations~\cite{qi}, and intermediate equivariant mappings~\cite{zaheer} to produce estimators that are invariant to permutations. While these architectures are asymptotically universal, they are notably limited in the intra-set dependencies they model. For instance, processing each set element 
independently using an MLP and aggregating features via max pooling to produce a set-level feature representation, as proposed in \cite{qi}, may have difficulty capturing pairwise relations among the set elements. This indicates that the effectiveness of such methods may be limited for complex real-world data with finite samples. Hence, we propose to build more expressive architectures by explicitly incorporating relational learning on elements of input sets, a proven strategy for better learning \cite{santoro}. Although not the first work that considers relational information for modeling set-structured data, as shown in our experiments, our approach learns to capture, and is capable of representing, the relational structure of each input set in an intuitive way, allowing for a deeper understanding of the data in hand while still leading to more effective learning results. 
\smallskip



\noindent \textbf{Main Contributions}\ \ \  In this paper, we further relational learning on sets with our framework \textit{Deep Message Passing on Sets} (DMPS). Our main contributions are: 1) we unite learning on graphs with learning on sets through deep kernel learning, allowing for flexible relational learning on sets; 2) we develop two novel blocks, the \textit{set-denoising block} and the \textit{set-residual block}, to further facilitate learning interactions among set elements; 3) in addition to demonstrating the interpretability of our model by successfully learning the true underlying relational structures, we show the effectiveness of our approach on both synthetic and real-world datasets by achieving results that are competitive with or outperform the state-of-the-art.

\section{Background}
\label{sec:background}

This section introduces background material relevant for the development of our DMPS approach.

\subsection{General Formulation of Valid Set Functions}
\label{subsec:general_formulation_of_valid_set_functions}

Given an input set $X = \{x_1, x_2, \cdots, x_n\}$, most permutation invariant functions that operate on sets studied in recent literature~\cite{zaheer,qi,max} belong to the following class of functions: 
\begin{equation}
\mathbb{F}(f) = \{f: f(X) = \kappa \left(pool\{\phi_X(S_1), \dots, \phi_X(S_m)\}\right)\}\,,
\label{eq:permutation_invariant_fcn_class}
\end{equation}
where $S_i$ for $1\leq i \leq m$ is a subset of the power set of X, $m$ is the number of such subsets being modeled, $\phi_X$ is a function that acts on the power set of $X$, $\textit{pool}$ is a permutation invariant pooling operation that produces a set-level representation for downstream tasks, and $\kappa$ is another function that transforms the set-level representation into the output space of the model. Most methods in the literature differ in the choice of $\phi_X$. 
To explicitly encode interactions among set elements, we will construct a function $\hat{\phi}_X$ that acts on a nontrivial subset of $X$ using a message passing scheme. 

\subsection{Message Passing on Graphs}
Representation learning on graphs is an active research area \cite{ham}. We focus on the message passing scheme. Consider a graph $G = \{V,E\}$, where $V$ is the set of vertices and $E$ is the set of edges. We assume that each node $v$ has a node feature $h_v$ and each edge between two vertices, say $v$ and $w$, has an edge feature $e_{vw}$. Message passing on graphs can be summarized in terms of a \textit{message passing phase} and a \textit{readout} phase~\cite{gilmer}. In the message passing phase we update each $h_v^t$ as 
$$m_v^{t+1} = \sum_{w\in N(v)} M_t(h_v^t, h_w^t, e_{vw})\ \text{(propagate neighborhood)}$$
$$h_v^{t+1} = U_t(m_v^{t+1}, h_v^t)\ \ \text{(update the feature vector)}$$
where $M_t$ and $U_t$ are the feature aggregation and update functions, respectively, $N(v)$ denotes the neighborhood of $v$, and $t$ enumerates the message passing step. In the \textit{readout phase}, we produce a graph-level feature representation from the node features for downstream tasks with $\hat{f}$ as
$$\hat{f} = R\left(\{h(v)|v\in V\}\right) $$ 
where $R$ is the readout function. To ensure the message passing scheme is permutation invariant with respect to graph isomorphisms, an example of a trio of $M_t$, $U_t$, and $R$ can be a function that computes a weighted average of the neighborhood features based on the edge weights, a concatenation operator, and a sum operator, respectively.

\section{Deep Message Passing on Sets}
\label{sec:deep_message_passing_on_sets}

We bridge learning on graphs and learning on sets with the message passing scheme, by first learning an underlying latent graph that represents the connectivity of the set elements, and then applying message passing to this latent graph to incorporate relational learning into learning on sets (see Fig.~\ref{fig:dmps_overview} for an illustration). In this way, DMPS leverages relational information to encode input sets in contrast to more traditional approaches \cite{zaheer,qi}, where each set element is either processed through some rudimentary equivariant transformations or in an independent manner. 

\subsection{Latent Graph Learning}
\label{subsec:latent_graph_learning}

Message passing on graphs is based on neighborhood structures and edge weights of the graphs. Our goal is to leverage relational information to encode elements in an input set via message passing, a natural way to capture the intra-dependencies among the elements if a graph that appropriately underpins the input set exists. However, unlike graph data, we generally do not know a-priori what neighbors a particular set element has and how strongly it is connected to these neighbors. Instead we need to infer such a graph structure from the set itself, ideally in an end-to-end fashion that is optimized jointly with the message passing scheme and a downstream task objective. To this end, we propose to learn edge weights among set elements, say $x_i$ and $x_j$, via a similarity function, $e_{i,j} = s(x_i,x_j)$, where the similarity function itself is learned through the \emph{deep kernel learning} scheme~\cite{wilson}. 
Specifically, deep kernel learning uses a shared multilayer perceptron network (MLP) to transform set elements into a feature space in which a kernel function is applied to the resulting feature representations, i.e.  $s(x_i,x_j) = k_{\sigma}\left(\text{MLP}(x_i), \text{MLP}(x_j)\right)$ where $k_{\sigma}$ is a valid kernel function with an adaptive hyperparameter $\sigma$. Following this kernel strategy, we effectively use an infinite number of adaptive basis functions to estimate the similarities among set elements. 

\subsection{Message Passing on Sets}
\label{subsec:message_passing_on_sets}

We simplify message passing on graphs to adapt to our setting: 
\begin{Def}
Given a set $X = \{x_1, x_2, \dots, x_n\}$ where $x_i \in R^p$ for all $i$, we define message passing on sets as an iterative updating procedure that updates each set element as a weighted sum of the entire set $x_i \longleftarrow \sum_{j=1}^n w_{i,j} x_j$,\ where $\sum_{j} w_{i,j} = 1$,\ $\forall i: \ 1 \leq i \leq n$, and $w_{i,j} \geq 0$,\ $\forall i,j: \ 1 \leq i,j \leq n$. 
\end{Def}
\noindent More compactly, we have $X^{t+1} = W X^{t}$ where $t$ denotes the time step, $X^{t+1}, X^{t} \in R^{n\times p}$, and $W \in R_+^{n\times n}$ is an adaptive, row-normalized stochastic matrix constructed using the deep kernel learning scheme. More specifically, given an input set, we first construct the kernel matrix $K$ where $K_{i,j} = s(x_i,x_j)$ using the learned similarity function. We then obtain $W$ by applying the \verb+Softmax+ operator to $K$ to ensure the rows of $K$ sum up to one, allowing us to interpret $W$ as the weighted, row-normalized adjacency matrix of the underlying (fully-connected) latent graph. If one possesses prior knowledge about the set elements, choosing an appropriate threshold, say $\delta$, and setting the entries of $W$ that are smaller than $\delta$ to zero would result in a sparser graph. Furthermore, if one were to stack, for example, $k$ message passing steps (i.e, $1 \leq t \leq k+1$), the estimated weight matrix, $W$, of the learned latent graph is shared across the stacked $k$ steps. This is to say, although $W$ is jointly estimated with other parameters, it is not step-specific, i.e., the underlying latent graph is assumed to be static. Stacking multiple message passing steps thus can be interpreted as propagating information from each set element's multi-hop neighbors to encode higher-order information. Next, we develop some concepts that are needed to introduce the set-denoising block.


\subsection{Diffusion on General Graphs}
\label{subsec:dmps_as_diffusion_on_graphs}

We present the update equation that explains message passing and motivates the set-denoising block from a diffusion point-of-view. We refer interested readers to the supplementary material for a more detailed treatment. We first point out that the Dirichlet energy can be discretized for graphs as
\begin{equation}
E(\{x_i\}) = \frac{C}{2} \sum_{(i,j)\in E} w_{i,j}||x_i - x_j||_2^2,
\label{eq:discrete_dirichlet}
\end{equation}
where $x_i$ denotes the feature vector of node $i$, $w_{i,j}$ is the weight of the edge connecting nodes $i$ and $j$, and $C$ is a constant. Differentiating Eq.~\eqref{eq:discrete_dirichlet} with respect to $x_i$, discretizing time with an Euler-forward approximation, and rearranging the terms, we obtain
\begin{eqnarray}
    x_i^{t+1} &=& (1-\delta t C \sum_j w_{i,j})x_i^t + \delta t C \sum_j w_{i,j} x_j^t,\\
    &=& (1-\delta t C)x_i^t + \delta t C \sum_j w_{i,j} x_j^t.\label{eq:discrete_anisotropic_heat_equation}
\end{eqnarray}
where $\delta t$ denotes the time step.

Hence, the update rule is a convex combination of the current node feature $x_i^t$ and the update feature obtained by the message passing step. Subject to time-step constraints on $\delta t$, larger $\delta t$ will result in a smoother solution. If we choose the constant $C$ and the time step $\delta t$ such that $C\delta t = 1$, we recover the message passing step. This observation offers us another interpretation of DMPS: message passing with the weight matrix $W$ is equivalent to diffusing each set element based on the entire set, i.e, updating each set element by a weighted average of the entire set. 


\subsection{Set-denoising and Set-residual Blocks}
\label{subsec:denoising_and_residual_blocks}

Recall that beyond first-order relations, stacking a desired number of message passing steps allows us to take higher-order interactions among set elements into account. While enhancing the model's capability of capturing complex relational signals, such stacking results in deeper networks, which can potentially cause problems in terms of over-smoothing the feature representations and other common difficulties in training deep networks.  
Based on the previously derived update equation~(\ref{eq:discrete_anisotropic_heat_equation}) and the residual network \cite{he}, we propose to address both concerns by introducing the set-denoising block and the set-residual block. 
Before detailing the architectures of those two blocks, we elaborate on the intuitions behind them:
\begin{itemize}
    \item  It is well-known that a discretized diffusion process with Neumann boundary conditions on a graph converges to a steady state (i.e all node features being the same) with enough time steps. In our setting, by stacking $n$ message passing steps we effectively are ``running" the discretized diffusion process on the latent graph over $n$ time-steps. Although accounting for interactions among set elements is crucial, it is also indispensable to retain meaningful distinctions among them (e.g, all set elements sharing the same feature representation when the diffusion process has converged is obviously not ideal for learning). It is thus important to avoid over-smoothing when stacking message passing steps, a critical observation that is attested by one of our experimental studies. 
    \item Ideally, the architectures of either the set-denoising block or the set-residual block should alleviate the concerns of vanishing gradient or difficulty of learning the identity mapping when building a deep network.
\end{itemize}


\begin{algorithm}[!t]
\SetAlgoLined
\KwResult{Final label prediction $\hat{y}$;}
\textbf{Input}: $\{x_1,x_2,\dots,x_n\}$: a set of objects; \textit{k}: the number of set-denoising blocks desired\;
\textbf{Learnable Parameters}: $H_t$, $\forall t:1\leq t \leq k$: parameters of the linear layers; $\gamma$: the diffusuion coefficient\;
\textbf{Initialization}: $t = 1$\;
---Extract $p$-dimensional numerical feature of each set element $x_i$ and form a feature matrix $X\in R^{n\times p}$\; 
---Construct the weight matrix $W$ from the feature matrix $X$ using deep kernel learning\;
    \While{$t\leq k$}{
        $X^{t+1} = (1-\gamma) X^t + \gamma WX^t$\;
        $X^{t+1} = \tau\left(X^{t+1} H_{t}\right)$\;
        $t = t + 1$\;
    }
$x = \verb+pool+(X^k)$\;
$\hat{y} = \kappa(x)$\
\caption{Deep Message Passing on Sets with the Set-denoising Block}
\label{alg:dmps}
\end{algorithm}

Now we formally introduce the two proposed set blocks

\begin{itemize}
    \item \textbf{Set-denoising Block:} As depicted in Fig.~\ref{fig:denoising}, we first apply message passing to the feature matrix $X$, resulting in $X^{\text{MP}}$. The original feature matrix $X$ is then combined with $X^{\text{MP}}$ via a convex combination. This convex combination is based on a learnable diffusion coefficient $\gamma\in (0,1)$, which 
    corresponds to $\delta t C$ in the discretized anisotropic heat equation~\eqref{eq:discrete_anisotropic_heat_equation}. 
    In other words, we do not explicitly choose the constant $C$ and the step size $\delta_t$ in~\eqref{eq:discrete_anisotropic_heat_equation}; instead we either fix $\gamma$ upfront or jointly learn it with other model parameters. It is worth noting that if $\gamma$ is learnable, we effectively allow our model to adaptively adjust the optimal degree of smoothing. We choose $\gamma$ to be the same for each block, though making $\gamma$ block-specific is possible. The linear layer followed by a non-linear operator at the end serve to further increase the expressiveness of the learned features. 
    
    \item \textbf{Set-residual Block:} Fig.~\ref{fig:residual} is directly motivated by the residual network~\cite{he}: assuming there exists an optimal feature matrix for learning, $X^{\text{optim}}$, it might be easier for the network to learn the difference, $X^{\text{optim}} - X$, through message passing as opposed to learn the optimal feature matrix $X^{\text{optim}}$ from scratch. Moreover, the architecture of the set-residual block alleviates some common problems that come with training a deep network~\cite{he}.
\end{itemize}

\subsection{Deep Message Passing on Sets (DMPS)}
\label{subsec:dmps}

As a concrete example, Algorithm~\ref{alg:dmps} outlines \textit{Deep Message Passing on Sets} (DMPS) coupled with the set-denoising block, where $\tau$ is an element-wise non-linear operator, \verb+pool+ is a set pooling operator, and $\kappa$ is a function (e.g a fully-connected layer) that transforms the set-level feature representation into the output space. 

\subsection{Analysis}
\label{sec:analysis}

Notice that with or without the set-denoising/set-residual blocks, the message passing step in DMPS is permutation equivariant with respect to the set elements. Since the composition of a permutation equivariant operation and a valid set pooling operation 
 is permutation invariant, we have the following proposition
\begin{prop}
DMPS is permutation invariant to the order of the elements in the input set. 
\end{prop}

\noindent Additionally, being able to approximate any valid set function is desirable for a model. Building upon the results in \citet{zaheer} and \citet{qi}, we have

\begin{prop}
DMPS is an universal approximator for any permutation invariant function. 
\end{prop}
\begin{proof}
See supplementary material.
\end{proof}

\begin{figure*}[!t]
\centering
\begin{subfigure}{.29\textwidth}
    \centering
    \includegraphics[width=1\linewidth]{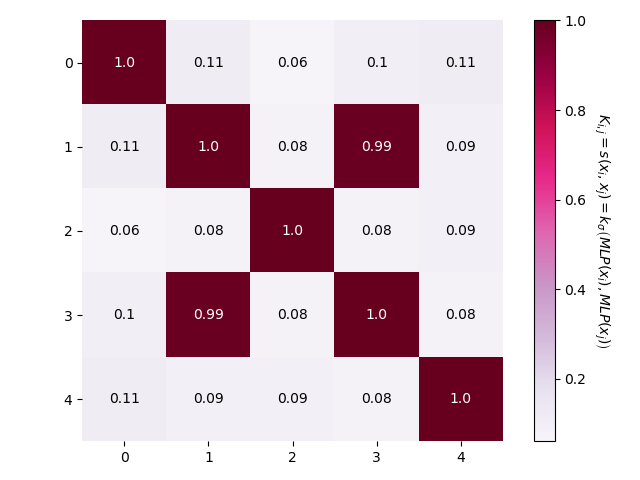}
    \caption{Learned kernel matrix---$N(\textbf{0}, \Sigma$)}
    \label{fig:sigma}
\end{subfigure}%
\begin{subfigure}{.29\textwidth}
    \centering
    \includegraphics[width=1\linewidth]{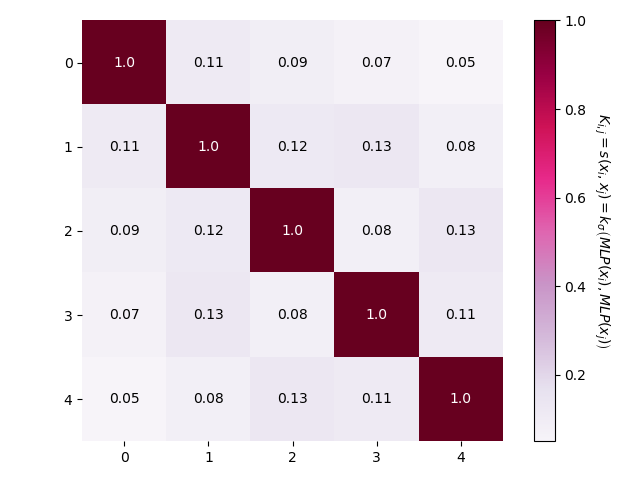}
    \caption{Learned kernel matrix---$N(\textbf{0},\textbf{I})$}
    \label{fig:identity}
\end{subfigure}
\begin{subfigure}{.29\textwidth}
    \centering
    \includegraphics[width=1\linewidth]{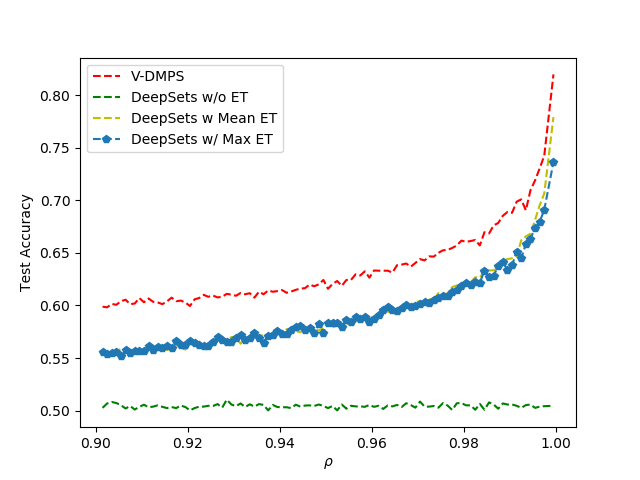}
    \caption{Test results with different $\rho$}
    \label{fig:result}
\end{subfigure}
\caption{Fig. (a) and (b) show the learned kernel matrices with input sampled from the respective two normal distributions at test time, and Fig. (c) depicts how the testing accuracy varies with $\rho$.}
\label{fig:weight_rec}
\end{figure*}

\section{Related Work}
\label{sec:related_work}


\subsection{Relational and Non-local Reasoning}
\label{subsec:relation_networks}

\begin{itemize}
    \item \textbf{Relational Reasoning}: 
    Our proposed approach fits and generalizes the relation network~\cite{santoro}. To elaborate, 
    if we do not perform any feature transformation after message passing and define the set pooling operator as the sum operator, a model comprised of one message passing step would result in $f(X) = \kappa\left(\sum_{i,j} w_{i,j}x_j^T H_1\right)$ where $x_j^T$ is the $j$-th row of the feature matrix $X$. This agrees with the form suggested in~\citet{santoro} if we let $g_{\theta}(x_i,x_j) = w_{i,j} x_j^T H_1$. Therefore, our proposed network can be seen as a generalization of the relation network in the sense that stacking multiple message passing steps or set blocks enables us to learn high-order relations among the set elements. 
    \item \textbf{Non-local Networks}: While non-local relational reasoning, whose main objective is to extend the frameworks of local learning schemes like the convolution operator or the fully-connected layer to allow for non-local learning, has been proposed in earlier work \cite{wang}, DMPS is an extension of that to set-structured data where our goal is to leverage non-trivial relations among set elements. We also proposed a new way of learning the similarity function, $s(x_i,x_j)$, that infers pairwise relations, namely \textit{deep kernel learning}. Last but not least, in order to further strengthen the flexibility of our model, we coupled DMPS with the set-denoising and set-residual blocks, motivated by the diffusion dynamics and the residual network, respectively, whereas \citet{wang} only considered residual connections. 
\end{itemize}

\subsection{Learning Deep Networks} 

The network proposed by \citet{he} allows for sensible training of deep networks, an advantage that is inherited by the set-residual block. 
While the set-denoising block shares certain structural similarities with the highway network~\cite{srivastava}, it was introduced in a different context, namely through the lens of diffusion on graphs and the avoidance of over-smoothing. In retrospect, the connection between diffusion dynamics and the set-denoising block established in this work provides another useful interpretation of the highway network in terms of relational learning. Last but not least, the set-denoising and set-residual blocks are introduced specifically to handle set-structured data, an attribute that is beyond the scope of the residual network or the highway network.

\subsection{Learning on Latent Graphs of Sets}
\label{subsec:learning_on_latent_graphs}

We now refer back to Eqn.~\ref{eq:permutation_invariant_fcn_class}. A special case of $\phi_X$, as it can be seen in \citet{qi} for example, acts on set elements independently. Notwithstanding, this can still be considered a trivial form of message passing on a latent graph whose edge set is empty, including the approaches in \citet{qi} and \citet{zaheer} as special cases. The set transformer \cite{lee}, by directly applying the transformer~\cite{Vaswani} on set-structured data, essentially uses a different $W$ for every message passing step, and there are effectively many different $W$'s at each step, since each attention head uses different keys/values. Unlike the transformer, DMPS places sensible restrictions on $W$ that are well-motivated: they are based on physics (diffusion dynamics) and graph learning (static latent graph), easier to analyze/visualize because there is only a single W, and have fewer free parameters, which means DMPS might generalize better given limited data and explains the experimental advantage shown by DMPS compared to the set transformer. 


\section{Experiments}
\label{sec:experiments}

We apply DMPS and its extensions to a range of synthetic-toy and real-world datasets. For each experiment, we compare our methods against, to the best of our knowledge, the state-of-the-art results for that dataset. Unless otherwise specified, three message passing steps, set-denoising blocks, or set-residual blocks are stacked to form the final model. 
Furthermore, we adopt the following abbreviations in this section: 
ET---equivariant transformation \cite{zaheer}; V-DMPS---vanilla DMPS, i,e the block component being the message passing step; R-DMPS---the block component being the set-residual block; D-DMPS w/(FDC or LDC)---the block component being the set-denoising block with (fixed or learnable, respectively) diffusion coefficient; and UG---uniform graph, i.e a fully-connected, undirected graph with equal weights for all edges.

\subsection{Classifying Gaussian Sets}

In this experiment we investigate a traditional problem of classifying random samples drawn from two different multivariate Gaussian distributions with the same mean and different covariance matrices. We create sets of real numbers by drawing the set elements as vector random samples from one of the two Gaussian distributions. The latent graph underlying each set is thus determined by the covariance matrix of its corresponding Gaussian distribution.
To use the covariance matrix as the ground truth for the latent graph, our main goal is to test DMPS's ability to capture and recover the true relational structure that underpins each set. Given a set $\textbf{X} = [X_1, X_2, \dots, X_p]^T$ where $\textbf{X} \sim N(\mu, \Sigma)$, it is worth emphasizing that the order among the elements in $\textbf{X}$ does not matter, as $\textbf{X}$ is permutation equivariant with respect to its mean and covariance. The canonical order is used here for convenience, although the result would be the same if one were to permute the order upfront. We next describe the experiment in more detail.

We sample input sets from two 5-dimensional Gaussian distributions $N(\mathbf{0},\textbf{I})$ and $ N(\mathbf{0},\Sigma)$.
To further test DMPS's capability to apprehend sparse relational signals, we choose $\Sigma$ to be the same as the identity matrix except at a (randomly) chosen pair of indices $\left((2,4), \text{in this case}\right)$ in which $\Sigma_{2,4} = \Sigma_{4,2} = \rho$ and $\rho\in [0,1)$. This is to say, the only, yet subtle difference between sets drawn from those two distributions is that the values of the second and the fourth elements
are positively correlated for sets drawn from $N(\textbf{0},\Sigma)$, a relational information that can only be captured if the elements in the sets are modeled interactively. 
Fig.~\ref{fig:sigma} and~\ref{fig:identity} showcase the latent graphs recovered by DMPS at test time with input sets sampled from the two chosen distributions, respectively, and with $\rho = 0.95$, while Fig.~\ref{fig:result} conveys how the test results vary with different choices of $\rho$.  

We have shown the following advantages of DMPS through this experiment: a). DMPS is able to take advantage of non-trivial covariance relations, thus outperforming methods that do not explicitly take such information into account. Furthermore, the trend of the curve in Fig.~\ref{fig:result} confirms the intuition that DMPS performs better when the underlying relational signal gets stronger; b). in contrast to other relational learning methods that focus on sets like \citet{lee}, DMPS is intuitively interpretable in that the kernel matrix learned through the \textit{latent graph learning} recovers the covariance structure of the underlying Gaussian distribution. This property of DMPS is highly desirable, and is also consistent with theoretical probability as functional transformations of random variables preserve the independence and covariance relations among those variables (the network acts on each set element, i.e each dimension of the Gaussian random sample, independently before the message passing step).

\subsection{Counting Unique Characters}
\label{subsec:counting_unqiue_characters}
\begin{figure}[!h]
\centering
\includegraphics[width=0.7\linewidth]{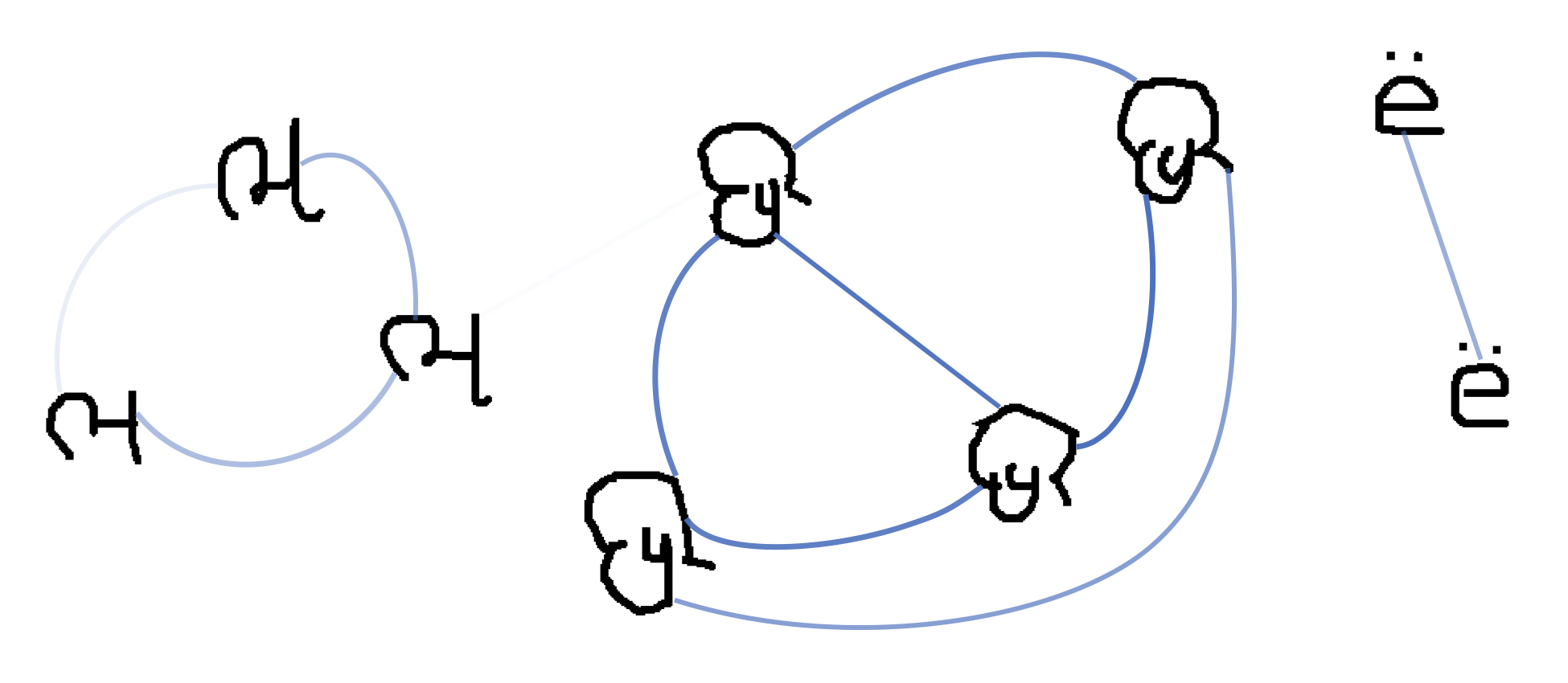}
\caption{Learned latent graph for a test input set with nine images categorized by three unique characters. The color transparency of the edges is proportional to their learned weights, with heavier-colored edges carrying larger weights.}
\label{fig:char_graph}
\end{figure}
To test the model's ability to model set-structured data relationally, \citet{lee} proposed the task of counting unique characters using the characters dataset \cite{lake}, where the goal is to predict the number of unique characters in an input set of character images. 
Please refer to the supplementary material for detailed experimental setup. Tab.~\ref{tab:counting} shows the testing results. We emphasize that we align as much architectural choices, such as learning rate, number of training batches, batch size, etc., as we can with \citet{lee} for fair comparison. We make some additional comments below.

\begin{table}[h]
  \begin{center}
    \caption{Counting unique characters}
    \label{tab:counting}
    \begin{tabular}{l|c} 
      \textbf{Architecture} & \textbf{Test Accuracy}\\
      \hline
      DeepSets w/ Mean ET & 0.4617 $\pm$ 0.0076 \\
      DeepSets w/ Max ET & 0.4359 $\pm$ 0.0077\\
      Set Transformer &  0.6037 $\pm$ 0.0075  \\
      \hline
      V-DMPS w/ UG &  0.1357 $\pm$ 0.0000\\
      D-DMPS w/ LDC \& UG & 0.4661 $\pm$ 0.0085\\
      \hline
      V-DMPS & 0.6446 $\pm$ 0.0174\\
      R-DMPS & 0.6600 $\pm$ 0.0103\\
      D-DMPS w/ FDC & \textbf{0.6748 $\pm$ 0.0120}\\
      D-DMPS w/ LDC & 0.6674 $\pm$ 0.0080\\
    \end{tabular}
  \end{center}
\end{table}

Firstly, DMPS and its variants outperform other methods by significant margins, showing the effectiveness of our proposed model. Secondly, since relational learning is the centerpiece of our model, we perform two ablation studies (V-DMPS w/ UG and D-DMPS w/ LDC \& UG) where we perform message passing on the uniform graph instead of the learned latent graph. The idea is that regular DMPS and its variants should outperform models with a fixed uniform graph if the learned latent graph indeed captures useful relational information. As shown in Tab.~\ref{tab:counting}, V-DMPS w/ UG performs poorly, while D-DMPS w/ LDC \& UG does better (perhaps because of the less erroneous smoothing induced by the set-denoising block) it still does not perform as well. This shows the significance of proper relational learning provided by regular DMPS and its variants. To further demonstrate the interpretability of our approach, Fig.~\ref{fig:char_graph} shows the learned latent graph for a test input set with nine images categorized by three unique characters. We see that the learned latent graph in which three clusters emerge delineates the relations among the set elements in a reasonable manner, underscoring a straightforward intuition that images corresponding to the same character are more closely related.

\subsection{Point Cloud Classification}
\label{subsec:point_cloud_classification}

\begin{table}[h!]
  \begin{center}
    \caption{ModelNet 40 Classification Task}
    \label{tab:ModelNet}
    \small
    \begin{tabular}{l|c|c} 
      \textbf{Architecture} & \textbf{100 points} & \textbf{1000 points}\\
      \hline
      DeepSets w/ Max ET & 0.82 $\pm$ 0.02 & 0.87 $\pm$ 0.01\\
      Set Transformer & 0.8454 $\pm$ 0.0144  & 0.8915 $\pm$ 0.0144\\
      \hline
      PointNet++ & --- & \textbf{0.907 $\pm$ ---}\\
      \hline
      V-DMPS & 0.8367 $\pm$ 0.0047 &  0.8751 $\pm$ 0.0029\\
      R-DMPS & 0.8475 $\pm$ 0.0036 &  0.8935 $\pm$ 0.0016\\
      D-DMPS w/ FDC & 0.8564 $\pm$ 0.0031 & 0.8783 $\pm$ 0.0032\\
      D-DMPS w/ LDC & \textbf{0.8571 $\pm$ 0.0062} & 0.8798 $\pm$ 0.0020\\
    \end{tabular}
  \end{center}
\end{table}

We apply DMPS and its variants to the ModelNet40 dataset~\cite{chang}, which contains objects that are represented as sets of 3d points (point clouds) in 40 different categories. Our model allows us to directly model sets of 3d points. 
For this experiment, we construct input sets with sizes 100 and 1,000 points per set by uniformly sampling from the mesh representations of the objects.
Tab.~\ref{tab:ModelNet} shows the test performances of our approaches compared with other state-of-the-art methods that directly operate on raw point clouds. We make some additional comments below. 

\begin{figure}[h]
\centering
\includegraphics[width=1.07\linewidth]{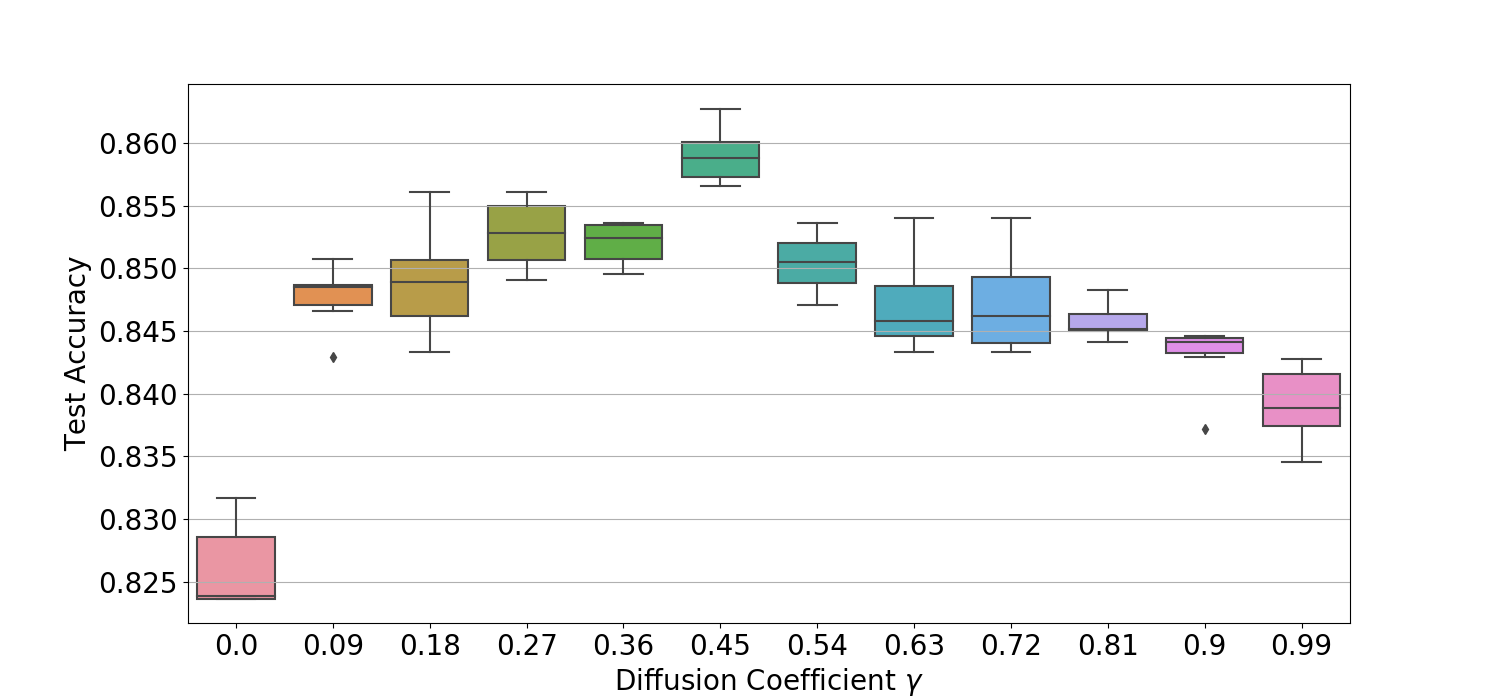}
\caption{Depiction of how the diffusion coefficient $\gamma$ affects the test accuracy in the case of 100 points per set.}
\label{fig:ts_study}
\end{figure}

Firstly, we point out that the task is harder when there are fewer points in the input set, thus requiring more efficient relational learning. We observe that D-DMPS w/ LDC outperforms other methods by significant margins in the case of 100 points in the input set. As for the case of 1000 points, PointNet++ achieved the state-of-the-art among methods that directly process raw point clouds. Our method performs on par with the set transformer \cite{lee}, and outperforms deep sets \cite{zaheer} by leveraging relational information. Secondly, to investigate the importance of balancing between appropriate message passing and over-smoothing, we perform a study in which we fix the diffusion coefficient to various values and see how the test accuracy varies. Fig.~\ref{fig:ts_study} shows the result. As $\gamma$ ranges from 0 to 1, the model effectively ranges from DeepSets w/o ET to V-DMPS, with anywhere in-between being D-DMPS w/ FDC at that particular $\gamma$. The test accuracy peaks when $\gamma$ approaches 0.5, and decreases when $\gamma$ becomes either too large or too small. This shows the significance of controlling the degree of smoothing. Along with the update equation, the novelty of the set-denoising block is affirmed theoretically and empirically.

\subsection{Histopathology Dataset}
\label{subsec:histopathology dataset}

\begin{table}[h!]
  \begin{center}
    \caption{Breast Cancer}
    \label{tab:hist}
    \begin{tabular}{l|c|r} 
      \textbf{Architecture} & \textbf{Test Accuracy}\\
      \hline
      Attention & 0.745 $\pm$ 0.018\\
      Gated Attention & 0.755 $\pm$ 0.016\\
      \hline
      V-DMPS & 0.800 $\pm$ 0.023\\
      R-DMPS & 0.818 $\pm$ 0.029\\
      D-DMPS w/ FDC & \textbf{0.846 $\pm$ 0.019}\\
      D-DMPS w/ LDC & 0.836 $\pm$ 0.023\\
    \end{tabular}
  \end{center}
\end{table}

The concept of learning on sets also applies well to \textit{weakly-labeled data}. In this section we perform experiment on classifying weakly-labeled real-life histopathology images provided in the breast cancer dataset \cite{gelasca}. 
A common approach is to divide an image into smaller patches and think of the patches as a set of "small images" with a single label for the set. 

The breast cancer dataset introduced in \citet{gelasca} consists of 58 weakly-labeled 896$\times$768 H\&E images. An image is labeled malignant if it contains breast cancer cells; otherwise it is labeled benign. We follow a similar procedure to pre-process the images as in \citet{max}. We divide the images into 32$\times$32 patches, which results in 672 patches per set (i.e., per image). Furthermore, because of the small number of available images, we perform data augmentation at the training stage by randomly rotating and mirroring the patches. We point out that \citet{max} also randomly adjusted the amount of H\&E by decomposing the RGB color of the tissue into the H\&E color space. We compare the performances of DMPS and its variants to the \textit{attention} and \textit{gated attention} models introduced in \citet{max}. The testing results are shown in Tab.~\ref{tab:hist}. Despite the framework of multiple instance learning, and thus the attention scheme, being particularly suitable to computational histopathology \cite{kandemir,max}, we see that DMPS and its variants perform uniformly better by significant margins. 

\section{Conclusion and Future Work}
\label{sec:conclusions}

In this paper, we introduced DMPS, a set-learning scheme that explicitly takes interactions among set elements into account when modeling set-structured data. To the best of our knowledge, our model is the first attempt to incorporate relational learning on sets via message passing on the latent graph of the input set. We also proposed two variants of message passing on sets, the set-denoising block and the set-residual block. Although this is a step towards relational learning on sets, there are many possible extensions. For example, the message passing scheme can be interpreted as a gradient descent step based on the weighted Dirichlet integral with the functional two-norm. Would one, for example, discretize an energy of the form $\int w(x)\|\nabla u(x)\|_2~dx$ instead, we would obtain a form of weighted total-variation message passing. Hence, one interesting future work would be to derive a family of message passing algorithms by changing the functional two-norm to the more general functional p-norm and to explore the behavior of the resulting message passing schemes. 


\bibliography{bibfile.bib}
\bibliographystyle{aaai}

\section{Supplementary Material}

\subsection{Derivations for the Diffusion Update Equation}

To motivate the set-denoising block, we first illustrate denoising via diffusion, which is related to Gaussian smoothing in the continuum~\cite{weickert1998anisotropic}. In the continuum the Dirichlet energy
\begin{equation}
 E\left(u(x)\right) = \frac{C}{2}\int ||\nabla u(x)||_2^2\ dx\,,
\label{eq:dirichlet_energy}
\end{equation}
where $u$ is a multivariate function and $C$ is a constant has the Euler-Lagrange equation~\cite{troutman2012variational}: $0=-\Delta u$, where $\Delta$ is the Laplacian. A gradient descent solution is given by the heat equation $u_t(x,t)=\Delta u(x,t); u(x,0)=u_0(x)$, which for $t\rightarrow\infty$ minimizes the Dirichlet energy~\eqref{eq:dirichlet_energy}. Note that $u(x,T)$ is equivalent to smoothing the initial condition $u_0(x)$ with a Gaussian with variance $\sigma^2=2T$~\cite{weickert1998anisotropic}. Hence, this is a form of smoothing or denoising. 

To further comprehend and possibly extend message passing on sets, we illustrate it from a diffusion point-of-view. We first generalize the Dirichlet energy to an anisotropic energy of the form
\begin{equation}
E\left(u(x)\right) =  \frac{C}{2}\int w(x)||\nabla u(x)||_2^2\ dx,~w(x)\geq 0.
\label{eq:dirichlet_energy_anistropic}
\end{equation}
This energy can be discretized for graphs as
\begin{equation}
E(\{x_i\}) = \frac{C}{2} \sum_{(i,j)\in E} w_{i,j}||x_i - x_j||_2^2,
\label{eq:discrete_dirichlet}
\end{equation}
where we replaced the function $u(x)$ by discrete feature vectors $x_i$ defined for each node $i$; we also replaced the gradient by a finite difference. We assume that our graph is undirected and we do not double-count edges. Such an energy is related to anisotropic diffusion equations~\cite{weickert1998anisotropic} which have for example been widely investigated for edge-preserving image-denoising. 

The gradient of Eq.~\eqref{eq:discrete_dirichlet} is
\begin{equation}
    \frac{\partial E}{\partial x_i} = C \sum_{j:j\in N(i)} w_{i,j} (x_i-x_j). 
\end{equation}
If we assume that $w_{i,j}=0$ if $(i,j)\notin E$ then we can write the equivalent of the heat equation as
\begin{equation}
    (x_i)_t = -C \sum_j w_{i,j}(x_i-x_j),
\end{equation}
where the right hand side corresponds to the discrete weighted Laplacian. Discretizing time with an Euler-forward approximation, we obtain
\begin{equation}
    x_i^{t+1} = x_i^t - \delta t C \sum_j w_{i,j} (x_i^t-x_j^t),
\end{equation}
where $\delta t$ is the time step and which (as $\sum_j w_{i,j}=1$ for our normalized kernel) can be written as
\begin{eqnarray}
    x_i^{t+1} &=& (1-\delta t C \sum_j w_{i,j})x_i^t + \delta t C \sum_j w_{i,j} x_j^t,\\
    &=& (1-\delta t C)x_i^t + \delta t C \sum_j w_{i,j} x_j^t.\label{eq:discrete_anisotropic_heat_equation}
\end{eqnarray}

Hence, the update rule is a convex combination of the current node feature $x_i^t$ and the update feature obtained by the message passing step. Subject to time-step constraints on $\delta t$, larger $\delta t$ will result in a smoother solution (in analogy to $\sigma^2=2T$ for the heat equation above). If we choose the constant $C$ and the time step $\delta t$ such that $C\delta t = 1$, we recover the standard message passing step described before. This observation offers us another interpretation of DMPS: message passing with weight matrix $W$ is equivalent to diffusing each set element based on the entire set, i.e, updating each set element by a weighted average of the entire set. 


\section{Proof}
\begin{prop}
DMPS is a universal approximator of any permutation invariant function. 
\end{prop}
\begin{proof}
 Functions of the form $\rho\left(\text{sum}\phi(\cdot)\right)$ where $\rho$ and $\phi$ are MLP networks are universal function approximators for permutation invariant functions~\cite{zaheer}. In the extreme case in which the estimated weight matrix $W$ is the identity matrix, DMPS, in its original form (i.e without  denoising/residual blocks), with one stage of message passing is exactly of the form proposed in \citet{zaheer}.  
\end{proof}

\section{Experiment Details}
In this section, we explain the setup and detail the architectural choices made for the experiments conducted in this work. We first list our notations for some basic deep learning layers
\begin{itemize}
    \item FL($d_i,d_o,f$) denotes a fully-connected layer with $d_i$ input units, $d_o$ output units, and activation function $f$.
    \item Conv($c_i$, $c_o$, $k_1$, $f$, $O(k_2,s)$) denotes a convolutional layer with $c_i$ input channels, $c_o$ output channels, kernel size $k_1$, activation function $f$, and pooling operation $O(k_2,s)$ with another kernel size $k_2$ and stride $s$.
    \item  MP($p$) denotes the max-pooling operation with input feature dimension $p$.
\end{itemize}  

\textit{Deep kernel learning} layer (DKL), introduced in the spirit of \citet{wilson}, takes in the original feature matrix and outputs a normalized kernel matrix. More specifically, it consists of two fully-connected layers followed by a mechanism, which we term \textit{kernel evaluation with bandwidth} $\sigma$ $\left(\text{KE}(\sigma)\right)$, that evaluates the chosen kernel (in this case, an RBF kernel) pairwise at the transformed features, and forms a kernel matrix. The output of DKL is a \textit{normalized} (created by applying the \verb+SoftMax+ operator to the raw kernel matrix output by KE($\sigma$)) kernel matrix. Written compactly, DKL($d_i, d_m, d_o, f_1, f_2, \sigma$) = [FL($d_i,d_m,f_1$), FL($d_m,d_o,f_2$), KE($\sigma$), SoftMax], outputting a normalized kernel matrix. Here, the operations in the list define the sequence of the operations.

A \textit{message passing} layer (MP) performs a matrix multiplication operation between the weighted kernel matrix output by DKL and the original feature matrix, followed by a fully-connected layer FC($d_i, d_i, f_3$) (note that the last fully-connected layer does not alter feature dimension). It is denoted as MP($d_i, d_m, d_o, f_1, f_2, f_3, \sigma$). A \textit{set-denoising block} (SDB) with learnable diffusion coefficient $\gamma$ consists of a message passing layer, an adaptive denoising layer (weighted addition in practice), and a fully-connected layer FC($d_i, d_i, f_3$). It is compactly denoted as SDB($d_i, d_m, d_o, f_1, f_2, f_3, \sigma, \gamma$). SDB with fixed diffusion coefficient is equivalent to setting $\gamma$ to 1/2. Similarly, a \textit{set-residual block} (SRB) is denoted as SRB($d_i, d_m, d_o, f_1, f_2, f_3, \sigma$). We emphasize two additional architectural choices
\begin{itemize}
\item The kernel bandwidth $\sigma$ and the diffusion coefficient $\gamma$ are learnable parameters. In practice, they optimized jointly with the other parameters in the network and gradients are computed via back-propagation.
\item When stacking multiple set-denoising/set-residual blocks, the normalized kernel matrix remains the same throughout the stacked blocks. In other words, the estimated latent graph is assumed to be static.
\end{itemize}

\subsection{Classifying Gaussian Sets}
We create sets of real numbers by drawing the set elements as vector random samples from one of the two Gaussian distributions, and our goal is to classify drawn sets to their underlying distributions. The latent graph underlying each set is thus determined by the covariance matrix of its corresponding Gaussian distribution. In other words, given a set $\textbf{X} = [X_1, X_2, \dots, X_p]^T$ where $\textbf{X} \sim N(\mu, \Sigma)$, there exists a corresponding latent graph in which each element $X_i$ is a node and an edge exists between two elements $X_i$ and $X_j$ if and only if they are not independent $\left(\text{i.e}, \Sigma_{i,j} \neq 0\right)$. 

More concretely, we generate an input training/testing set by sampling from one of the two mutlivariate normal distributions, $N(0,\Sigma)$ and $N(0,\textbf{I})$, with 
\[ \Sigma = \left( \begin{array}{ccccc}
1 & 0 & 0 & 0 & 0\\
0 & 1 & 0 & \rho & 0\\
0 & 0 & 1 & 0 & 0\\
0 & \rho & 0 & 1 & 0\\
0 & 0 & 0 & 0 & 1
\end{array} \right)
,\ \ \ \  %
\textbf{I} = \left( \begin{array}{ccccc}
1 & 0 & 0 & 0 & 0\\
0 & 1 & 0 & 0 & 0\\
0 & 0 & 1 & 0 & 0\\
0 & 0 & 0 & 1 & 0\\
0 & 0 & 0 & 0 & 1
\end{array} \right)
\]
where $\rho$ is a hyperparameter that controls how correlated the second and the fourth elements of a set sampled from $N(0,\Sigma)$ are. The task is to classify the input sets to their underlying distributions.

We only use the message passing step as the building block for our model, as the main goal of conducting this experiment is to see if DMPS can learn the correct covariance structure that underlies an input set. We train our model using the Adam optimizer with an initial learning rate of $10^{-3}$, and the learning scheduler, \textit{ReduceLROnPlateau}, with factor 0.9 and patience 1. We train on 120,000 batches with batch size of 128, 64 of which are sampled from $N(0,\textbf{I})$ and the other 64 from $N(0,\Sigma)$. The detailed architecture of the model is presented below

\begin{center}
 \label{tab:gaussian}
 \begin{tabular}{||c||} 
 \hline
 \textbf{Overall Architecture}  \\ [0.5ex] 
 \hline\hline
  FL(1,32,ReLU)  \\ 
 \hline
  MP(32, 64, 128, ReLU, ReLU, ReLU, $\sigma$)\\ 
 \hline
  MP(32, 64, 128, ReLU, ReLU, ReLU, $\sigma$) \\
 \hline
  MP(32, 64, 128, ReLU, ReLU, ReLU, $\sigma)$\\
 \hline
  MP(32)  \\
 \hline
  FL(32,1,Sigmoid)\\ [1ex] 
 \hline
\end{tabular}
\end{center}

\subsection{Counting Unique Characters}
To test the model's ability to model set-structured data relationally, \citet{lee} proposed the task of counting unique characters using the characters dataset \cite{lake}, where the goal is to predict the number of unique characters in an input set of character images. The character dataset consists of 1,623 characters from various alphabets, with 20 images for each character. We first split all the characters along with their corresponding images into collections of training and testing images, with half of the images for each character in the training collection and the other half in the testing collection. An input set for training is generated by sampling between 6 to 10 images from the collection of training images (same procedure to generate a testing input set). We emphasize again that we align as much architectural choices, such as learning rate, number of training batches, batch size, etc., as we can with \citet{lee} for fair comparison. Next, we describe how to generate an input training/testing set, and detail the architectural choices made for this experiment. 

We generate a training/testing input set as follows. We first sample the set size, $n$, uniformly from the collection of integers, \{6,$\cdots$, 10\}, and then uniformly sample the number of unique characters, $c$, from \{1,$\cdots$, n\}. With the set size $n$ and the number of unique characters $c$ in hand, we sample $c$ characters from the training/testing collections of characters, and then randomly sample instances of the chosen characters from a multinomial distribution so that the total number of instances sums to $n$ and each chosen character has at least one representation in the resulting set.

Poisson regression is used for prediction, with the mode of the distribution being the output of our model. The loss function we optimize over is the log-likelihood of the Poisson distribution, i.e $\log (x|\lambda) = -\lambda + x \log(\lambda) - \log(x!)$. We train our model using the Adam optimizer with a constant learning rate of $10^{-4}$ for 200,000 batches with batch size of 32. As an example, the detailed architecture for DMPS with the set-denoising block and learnable diffusion coefficient is the following

\begin{center}
 \label{tab:counting}
 \begin{tabular}{||c||} 
 \hline
 \textbf{Overall Architecture}  \\ [0.5ex] 
 \hline\hline
  Conv$\left(1,10,3,\text{ReLU,MaxPool}(2,2)\right)$ \\ 
 \hline
  Conv$\left(10,10,3,\text{ReLU,MaxPool}(2,2)\right)$  \\
 \hline
  Conv$\left(10,10,3,\text{ReLU,MaxPool}(2,2)\right)$  \\
 \hline
  Conv$\left(10,10,3,\text{ReLU,MaxPool}(2,2)\right)$  \\
 \hline
  SDB($160, 256, 512, \text{Tanh, Tanh, Tanh}, \sigma, \gamma$)\\ 
 \hline
  SDB($160, 256, 512, \text{Tanh, Tanh, Tanh}, \sigma, \gamma$) \\
 \hline
  SDB($160, 256, 512, \text{Tanh, Tanh, Tanh}, \sigma, \gamma$)\\
 \hline
  SP(160)  \\
 \hline
  FL(160,1,---) + Exponential Operator \\ [1ex] 
 \hline
\end{tabular}
\end{center}

\subsection{ModelNet40 Classification}
We use the ModelNet40 dataset for our point cloud classification experiment. This dataset consists of 9,843 training and 2,468 testing data, each of which belongs to one of the 40 classes and is represented by a set of 3-dimensional points. Our task is to classify each set correctly to its label. For each data, we reproduce a subset of it with $n = 100, 1000$ points, randomly sampled from the original set, as input to our model. We randomly rotate and scale each subsampled set on the fly during training to improve the robustness of our model. We also normalize elements in each set to have zero mean and at most length one (i.e., we normalize the lengths of the set elements so that the element with the largest length in each set has length 1). We train our model using the Adam optimizer with an initial learning rate of $10^{-3}$, and we use the ReduceLROnPlateau with delay rate 0.9 and patience 2 as the scheduler. Again, as an example, the detailed architecture for DMPS with the set-denoising block and learnable diffusion coefficient is presented below

\begin{center}
 \label{tab:cloud}
 \begin{tabular}{||c||} 
 \hline
 \textbf{Overall Architecture}  \\ [0.5ex] 
 \hline\hline
  FL(3,512,Tanh) \\ 
 \hline
  SDB($512, 512, 1024, \text{Tanh, Tanh, Tanh}, \sigma, \gamma$)\\ 
 \hline
  SDB($512, 512, 1024, \text{Tanh, Tanh, Tanh}, \sigma, \gamma$) \\
 \hline
  SDB($512, 512, 1024, \text{Tanh, Tanh, Tanh}, \sigma, \gamma$)\\
 \hline
  SP(512)  \\
 \hline
  FL(512,40,Tanh)\\ [1ex] 
 \hline
\end{tabular}
\end{center}

\subsection{Histopathology Dataset}
In this experiment we are working with 58 weakly-labeled, 896$\times$768 H\&E images \cite{gelasca}, where our goal is to classify each image either as benign or malignant. During training/testing time, we divide each image into 32$\times$32 patches, thus treating each image as a set of 32$\times$32 image patches. This results in 672 elements per set. Furthermore, because of the small number of available images, we perform data augmentation at the training stage by randomly rotating and mirroring the patches. We train our model using the Adam optimizer with a constant learning rate of $10^{-4}$. The images are trained one by one, i.e, we use a batch size of 1. The following table depicts the detailed architecture for DMPS coupled with the set-denoising block and learnable diffusion coefficient

\begin{center}
 \label{tab:h&e}
 \begin{tabular}{||c||} 
 \hline
 \textbf{Overall Architecture}  \\ [0.5ex] 
 \hline\hline
  Conv$\left(3,36,4,\text{ReLU,MaxPool}(2,2)\right)$ \\ 
 \hline
  Conv$\left(36,48,3,\text{ReLU,MaxPool}(2,2)\right)$  \\
 \hline
  FL(1728,512,Tanh) \\
 \hline
  FL(512,512,Tanh)\\
 \hline
  SDB($512, 512, 1024, \text{Tanh, Tanh, Tanh}, \sigma, \gamma$)\\ 
 \hline
  SDB($512, 512, 1024, \text{Tanh, Tanh, Tanh}, \sigma, \gamma$) \\
 \hline
  SDB($512, 512, 1024, \text{Tanh, Tanh, Tanh}, \sigma, \gamma$)\\
 \hline
  SP(512)  \\
 \hline
  FL(512,1,Sigmoid) \\ [1ex] 
 \hline
\end{tabular}
\end{center}

\end{document}